\documentclass[runningheads]{llncs}

\usepackage[T1]{fontenc}

\usepackage[american]{babel}
\usepackage{mathtools} 
\usepackage{booktabs}
\usepackage{xcolor}
\usepackage{soul}
\usepackage{varwidth}
\usepackage{caption}

\usepackage{wrapfig}
\usepackage{lipsum}

\usepackage{bbm}
\usepackage{amsfonts}

\usepackage{amssymb}

\let\emptyset\varnothing

\usepackage{algorithm}
\usepackage{algpseudocode}

\usepackage{tikz}
\usepackage{bm}
\usetikzlibrary{shapes, arrows, positioning, calc}
\usepackage{graphicx}

\usepackage{subfiles}

\usepackage{latexsym}
\usepackage{booktabs}
\usepackage{enumitem}
\usepackage{color}

\usepackage{hyperref}

\usepackage{mparhack}
\newcommand{\marginsubsection}[1]{%
    \normalmarginpar\leavevmode\marginpar{\raggedleft \textbf{#1}}%
}
\newcommand{\revmarginsubsection}[1]{%
    \reversemarginpar\leavevmode\marginpar{\raggedleft \textbf{#1}}%
}

\begin{document}

\title{Explainable Evidential Clustering}
\author{Victor F. Lopes de Souza\inst{1, 2, 3, 4} \and
Karima Bakhti\inst{1, 4} \and
Sofiane Ramdani\inst{3} \and \\
Denis Mottet\inst{1} \and
Abdelhak Imoussaten\inst{1, 2}
}

\authorrunning{Lopes de Souza et al.}

\institute{Euromov Digital Health in Motion, \\ Univ. Montpellier, IMT Mines Alès, Montpellier, France\and
Euromov Digital Health in Motion, \\ Univ. Montpellier, IMT Mines Alès, Alès, France
\and
LIRMM, Univ. Montpellier, CNRS, Montpellier, France \and
CHU, Univ. Montpellier, Euromov DHM, Montpellier, France \\ ~ \\
\email{victor-fernando.lopes-de-souza@umontpellier.fr}
}
\maketitle


\begin{abstract}
Unsupervised classification is a fundamental machine learning problem. Real-world data often contain imperfections, characterized by uncertainty and imprecision, which are not well handled by traditional methods. Evidential clustering, based on Dempster-Shafer theory, addresses these challenges. This paper explores the underexplored problem of explaining evidential clustering results, which is crucial for high-stakes domains such as healthcare. Our analysis shows that, in the general case, representativity is a necessary and sufficient condition for decision trees to serve as abductive explainers. Building on the concept of representativity, we generalize this idea to accommodate partial labeling through utility functions. These functions enable the representation of "tolerable" mistakes, leading to the definition of evidential mistakeness as explanation cost and the construction of explainers tailored to evidential classifiers. Finally, we propose the Iterative Evidential Mistake Minimization (IEMM) algorithm, which provides interpretable and cautious decision tree explanations for evidential clustering functions. We validate the proposed algorithm on synthetic and real-world data. Taking into account the decision-maker's preferences, we were able to provide an explanation that was satisfactory up to 93\% of the time.

\end{abstract}

\section{Introduction}\label{sec:intro}

\textbf{Clustering} is a fundamental machine learning problem \cite{macqueenMethodsClassificationAnalysis1967} that aims to group similar objects while distinguishing different ones \cite{hansenClusterAnalysisMathematical1997}. As a core data analysis task, it reveals patterns and enables applications such as data compression, summarization, visualization, and anomaly detection \cite{xuSurveyClusteringAlgorithms2005}. Regarding clustering, two major challenges persist: \textbf{imperfection of input data} \cite{duboisRepresentationsFormellesLincertain2006} and \textbf{interpretability} \cite{carvalhoMachineLearningInterpretability2019}.

Real-world scenarios with imperfect data require \textbf{cautiousness} \cite{bengsPitfallsEpistemicUncertainty2022,angelopoulosUncertaintySetsImage2022,imoussatenCautiousClassificationBased2022,hullermeierQuantificationCredalUncertainty2022,nguyenReliableMulticlassClassification2018}, defined  as decision-makers' awareness of model limitations and resulting risk-aversion. Effective cautiousness depends on properly characterizing these imperfections, primarily \textbf{uncertainty} and \textbf{imprecision} \cite{duboisRepresentationsFormellesLincertain2006}. Approaches addressing these issues include imprecise probability theory \cite{walleyStatisticalReasoningImprecise1991}, possibility theory \cite{duboisRepresentationCombinationUncertainty1988}, rough sets \cite{pawlakRoughSets1982}, fuzzy sets \cite{zadehFuzzySets1965}, and Dempster-Shafer belief functions \cite{shaferMathematicalTheoryEvidence1976}.

These foundations have given rise to various clustering methods, including fuzzy \cite{ruspiniNewApproachClustering1969}, possibilistic \cite{krishnapuramPossibilisticApproachClustering1993}, rough \cite{lingrasIntervalSetClustering2004}, and evidential clustering \cite{massonECMEvidentialVersion2008}. Regarding the latter, while classical hard clustering \cite{hartiganAlgorithm136KMeans1979} assigns each point to exactly one cluster, evidential clustering induces a credal partition \cite{massonECMEvidentialVersion2008} that represents both uncertainty and imprecision through partial membership across multiple cluster combinations.

Interpreting clustering results is equally critical. Without proper interpretation, clustering outcomes often lack practical utility, which has led to significant research in developing explainable clustering methods \cite{moshkovitzExplainableKMeansKMedians2020,bandyapadhyayHowFindGood2023,ben-hurSupportVectorClustering2001,ellisAlgorithmAgnosticExplainabilityUnsupervised2021,tutayClusterExplorerInteractiveFramework2023,ellisExplainableFuzzyClustering2024}. \textbf{Explainability} represents a rapidly growing field in machine learning, referring to a model's ability to provide details and justifications for its behavior in a manner that is clear and understandable to specific audiences \cite{barredoarrietaExplainableArtificialIntelligence2020}. Explainable systems are particularly valuable as they enable users to comprehend, critique, and improve models. While substantial research has focused on explaining black-box models \cite{ribeiroWhyShouldTrust2016,lundbergConsistentIndividualizedFeature2019}, with some researchers criticizing this emphasis \cite{rudinStopExplainingBlack2019}, increased attention is now directed toward systems that are explainable by design \cite{moshkovitzExplainableKMeansKMedians2020}.

There are multiple approaches to address explainability. A common objective is to find methods that can provide explanations for the behavior of existing models without requiring additional information about the model's structure. These are called model-agnostic techniques for post-hoc explainability \cite{carvalhoMachineLearningInterpretability2019}. They fall into two main categories \cite{barredoarrietaExplainableArtificialIntelligence2020}: feature relevance techniques that describe black-box model behavior by ranking or quantifying the influence of each feature on predictions \cite{lundbergConsistentIndividualizedFeature2019,baehrensHowExplainIndividual2010}, which have been applied to clustering \cite{alvarez-garciaComprehensiveFrameworkExplainable2024,ellisAlgorithmAgnosticExplainabilityUnsupervised2021,ellisExplainableFuzzyClustering2024}. As they provide no direct insights into the underlying dataset \cite{moshkovitzExplainableKMeansKMedians2020}, this have drawn criticism in high-stakes scenarios \cite{rudinStopExplainingBlack2019}; and \textbf{simplification} techniques that find more interpretable models approximating black-box classifiers, facing trade-offs between interpretability and accuracy, with approaches like LIME \cite{ribeiroWhyShouldTrust2016} restricting simplification to neighborhoods but potentially leading to incoherent explanations across different regions \cite{amgoudAxiomaticFoundationsExplainability2022}.

High-stakes domains like healthcare particularly require both explainability and cautiousness. However, explaining cautious clustering remains relatively unexplored despite its importance. As noted in some works \cite{zhangSurveyEvidentialClustering2024}, the development of explainability clustering methods over uncertain/imprecise data is worth exploring, especially for cautious and explainable approaches over imperfect data sources.

The main \textbf{objectives} of this paper are:
\begin{enumerate}
    \item To conduct a comprehensive investigation of decision trees as explainers for hard clustering functions, establishing conditions that define effective explanations.
    \item To develop a theoretical framework that extends these conditions to encompass uncertainty and imprecision, particularly within the evidential clustering paradigm.
    \item To introduce an innovative \textbf{Explainable Evidential Clustering} method through a novel algorithm grounded in these theoretical foundations.
\end{enumerate}

Our key \textbf{contributions} include:
\begin{enumerate}
    \item Building upon utility functions, we introduce the concept of \textit{Representativeness}, which quantifies decision-makers' preferences regarding errors committed by an explainer. This advancement enables systematic evaluation of cautious explanations.
    \item We propose the Evidential Mistakeness function and develop the \textbf{Iterative Evidential Mistake Minimization} (IEMM) algorithm. This novel approach, inspired by the IMM algorithm \cite{moshkovitzExplainableKMeansKMedians2020}, generates decision trees that effectively explain evidential clustering functions.
    \item We implement and validate this algorithm on both synthetic and real-world datasets, demonstrating how to select utility parameters that represent different decision attitudes.
\end{enumerate}

The remainder of this paper is structured as follows. Section~\ref{sec:background} presents the theoretical foundations, encompassing belief functions, evidential clustering, and explainability. Section~\ref{sec:algorithm} introduces utility functions, representativeness, evidential mistakeness, and the IEMM algorithm. We provide formal analysis of these concepts along with illustrative examples. We also make available the complete code for all experiments at \url{https://github.com/victorsouza89/iemm}.

\section{Background}\label{sec:background}
Let $X$ represent a set of \textbf{observations} in a known \textbf{feature space} $\mathbb{X}$. We assume $X = \{x_1, ..., x_N\} \subset \mathbb{X} = \mathcal{A}_1 \times \ldots \times \mathcal{A}_D$, where each element of $\mathcal{D} = \{\mathcal{A}_1, \ldots, \mathcal{A}_D\}$ is called an \textbf{attribute}. We assume all attributes are finite\footnote{This finite attribute assumption is crucial for decision tree operations. We may occasionally refer to $\mathbb{X}$ as $\mathbb{R}^D$ for simplicity, though this is an abuse of notation. When working with continuous attributes, we implicitly discretize the space (for example, with a dataset $X \subset \mathbb{R}^D$, we typically consider binary attributes $\mathcal{A}_{d, \theta} = \{True, False\}$ for each dimension $d \in \{1, ..., D\}$ and threshold $\theta \in \{x_d : x \in X\}$, where $x_{\mathcal{A}_{d, \theta}} = True$ if and only if $x_d \geq \theta$).}. In essence, $X$ is a set of $D$ measurements for each of $N$ objects, while $\mathbb{X}$ encompasses all possible measurements.

A classification problem is the task of assigning each observation in $X$ to an outcome from a finite set $\Omega = \{\omega_1, ..., \omega_C\}$, which we call the \textbf{frame of discernment}. A function that performs this assignment is called a \textbf{classifier}. When a set of training examples is available, we refer to the problem of constructing such function as \textbf{supervised classification}. In contrast, if no training examples are available and the goal is to group observations based on their similarity—without prior knowledge of the classes or labels—the task is called \textbf{unsupervised classification} or \textbf{clustering}.

\subsection{Belief Functions}\label{sec:belief-functions}
The Dempster-Shafer theory of evidence \cite{shaferMathematicalTheoryEvidence1976} provides a framework for representing uncertain and imprecise information. At the core of this theory lies the \textbf{mass of belief function}, or simply \textbf{mass function}—a map defined as:
$$m : 2^\Omega \rightarrow [0,1] \text{ such that } \sum_{A \subseteq \Omega} m(A) = 1.$$ Within this framework, an element $\omega \in\Omega$ represents the finest level of discernible information. The mass $m(A)$ quantifies the degree of confidence in the statement that 'the correct hypothesis $\omega$ belongs to $A \subseteq \Omega$, yet it remains impossible to determine which specific element of $A$ is correct'.

From the mass function, we derive two important measures—belief and plausibility functions:
$$
\operatorname{Bel}(A)=\sum_{B \subseteq A} m(B) \text{ and }
\operatorname{Pl}(A)=\sum_{B \cap A \neq \emptyset} m(B).
$$ The belief function $\operatorname{Bel}(A)$ for some $A \subseteq \Omega$ represents the degree of confidence that 'the correct hypothesis $\omega$ belongs to $A$'. In contrast, the plausibility function $\operatorname{Pl}(A)$ captures the degree of confidence that 'it is not impossible for the correct hypothesis $\omega$ to belong to $A$'.

Mass functions generalize probability mass functions by distributing belief across all subsets of $\Omega$ rather than just its individual elements. When $m(\emptyset) = 0$, we say the mass \textit{satisfies the closed-world hypothesis} \cite{smetsNonStandardLogicsAutomated1988}, meaning it rejects the possibility that the correct hypothesis $\omega$ lies outside $\Omega$.

We define the \textbf{focal set} of $m$ as $\mathbb{F}_m = m^{-1}(]0,1])$—the collection of all subsets of $\Omega$ assigned nonzero belief. Each member of this set is known as a \textbf{focal element}. Mass functions can be categorized based on their focal elements:

\begin{itemize}
    \item If all focal elements are singletons (of cardinality 1), then $p(\omega) = m(\{\omega\})$ forms a probability mass function, and $m$ is called a Bayesian mass function.
    
    \item A mass function with exactly one focal element $A$ is called \textit{categorical}, representing the logical assertion that '$\omega$ belongs to $A$'. If this single focal element is $\Omega$ itself, the function is \textit{vacuous}, conveying no information beyond the closed-world hypothesis.

    
\end{itemize}

We denote by $\mathbb{M}$ the set of all mass functions defined on $\Omega$. For notational simplicity, we may write $\omega_i \cup \omega_j \cup \omega_k \cup ...$ to represent the subset $\{\omega_i, \omega_j, \omega_k, ...\}$.

\subsection{Evidential Clustering}\label{sec:evidential-clustering}
\begin{figure*}[ht]
    \centering
    \includegraphics[width=\linewidth]{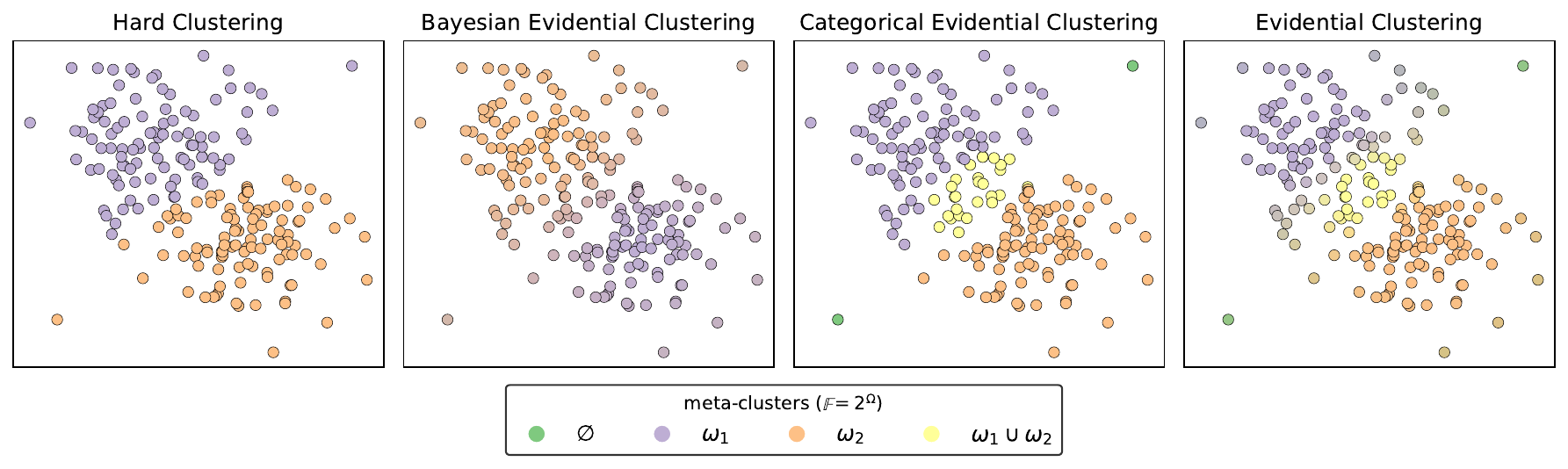}
    \caption{A representation of different clustering functions over a synthetic dataset. In this case, $\mathbb{X} = \mathbb{R}^2$ and $|\Omega| = 2$. Dataset was constructed by sampling 100 points from two normal distributions with centers at $(3, 5)$ and $(5, 3)$ and $\sigma$ of 1. Two outliers were added, at $(2, 2)$ and $(6, 6)$. The \texttt{evclust} package \cite{soubeigaEvclustPythonLibrary2025} was used to perform clustering. Hard clustering assigns each point to a single cluster. Bayesian evidential clustering gives a membership level for each observation. Categorical evidential clustering introduces the information about in-between points ($\omega_1 \cup \omega_2$) and outliers ($\emptyset$). Finally, evidential clustering combines all the previous information. The gradient of colors for each point visually represents the mass.}
  
    \label{fig:clustering-methods}

    \vspace{-10pt}
  \end{figure*}

\begin{definition}
  An \textbf{evidential clustering} is a map $\mathcal{M} : X \rightarrow \mathbb{M}$.
  
  For an observation $x \in X$, the function $\mathcal{M}(x)$, which we may denote as $m_x$, when evaluated at $A \subseteq \Omega$, returns the degree of confidence attributed to the statement 'the class $\omega$ corresponding to $x$ belongs to $A$, and it is not possible to determine which specific element of $A$ is the correct one'.
\end{definition}
We refer to each element of $\Omega$ as a cluster and each non-singleton subset of $\Omega$ as a metacluster. Within the context of an evidential clustering function, we denote $\mathbb{F}_{\mathcal{M}} = \bigcup_{x \in X} \mathbb{F}_{m_x}$. When all $m_x$ are categorical, we say that $\mathcal{M}$ is \textit{categorical}. Similarly, when all elements of $\mathbb{F}_{\mathcal{M}}$ are singletons, we call $\mathcal{M}$ \textit{bayesian}. An evidential clustering function that is both categorical and bayesian naturally induces a hard clustering. A \textbf{hard clustering} is simply a partition of observations into clusters, formalized as a surjection $\mathcal{C} : X \rightarrow \Omega$.

Figure \ref{fig:clustering-methods} illustrates various clustering methods for $\Omega = \{\omega_1, \omega_2\}$. While hard clustering assigns each point to exactly one cluster, evidential clustering offers a more sophisticated representation by capturing uncertainty and imprecision. It identifies points that may plausibly belong to multiple clusters (represented in the figure by those primarily associated with $\omega_1 \cup \omega_2$) and accounts for points not clearly associated with any cluster (shown as those predominantly linked to $\emptyset$).

Some clustering algorithms naturally produce a centroid for each cluster. A \textbf{centroid} $v_{\omega}$ (or $v_{A}$) is a point in the feature space $\mathbb{X}$ that represents its cluster $\omega \in \Omega$ (or metacluster $A \subset \Omega$).

For the remainder of this work, we assume $\emptyset \notin \mathbb{F}_{\mathcal{M}}$, rejecting the outlier hypothesis.

\subsection{Decision Trees Explaining Hard Clustering}\label{sec:explainability}
In this paper, we propose a simplification technique for explaining clustering. Such simplification techniques typically rely on rule extraction methods \cite{barredoarrietaExplainableArtificialIntelligence2020}. In this context, a particularly desirable outcome \cite{amgoudAxiomaticFoundationsExplainability2022} is what is known as an abductive explanation. Abductive explanations were introduced to address the question: "Why is $\Gamma(x) = \omega$?", providing a sufficient reason for characterizing the label $\omega$, where $\Gamma$ is a supervised classifier \cite{ignatievAbductionBasedExplanationsMachine2019}. In the following of this paper we will denote $\mathbf{C}$ the set of all consistent subsets of feature literals. More details on the definitions and formal aspects of explainers can be found in appendix \ref{app:simplification-explanation-techniques}.

\marginsubsection{Decision Trees as Explainers}
\textbf{Decision Trees} (DTs) \cite{quinlanSimplifyingDecisionTrees1987} are classical machine learning algorithms, classifiers based on rooted computation trees expressed as recursive partitions of the observation space \cite{rokachDecisionTrees2005}. Building upon these partitional aspects, we define a \textbf{node} of a decision tree as the subset $S \subseteq \mathbb{X}$ associated with its vertices. The nodes corresponding to the terminal vertices of the decision tree are called \textbf{leaves}.

Any decision tree $\Delta:\mathbb{X} \rightarrow \Omega$ can be transformed in linear time into a DNF expression \cite{audemardExplanatoryPowerBoolean2022a} (see appendix \ref{app:simplification-explanation-techniques}, equation \eqref{eq:DNF}). Due to this property, it induces an explainer $\chi_\Gamma^\Delta : \Omega \rightarrow 2^\mathbf{C}$. When $\Delta = \Gamma$, this explainer is representative and therefore provides abductive explanations. More details on the decision tree construction from the standpoint of explainability can be found in appendix \ref{app:dts-as-explainers}, along with the proof that representativity is a necessary and sufficient condition for decision trees to provide abductive explanations.

In this context, it is natural to assess the quality of explanations by the "representativeness" of the explainer $\chi_\Gamma^\Delta$. This assessment is typically performed by measuring the accuracy \cite{ribeiroWhyShouldTrust2016,izzaProvablyPreciseSuccinct2022,narodytskaAssessingHeuristicMachine2019} of the underlying classifier $\Delta$. Thus, the quality of the explanation provided by $\chi_\Gamma^\Delta$ about $\Gamma$ is quantified as:
\begin{equation}
    \operatorname{Accuracy}_\Gamma(\Delta) = \frac{|\{x \in X : \Gamma(x) = \Delta(x)\}|}{|X|}.
    \label{eq:precision_hard}    
\end{equation}

\marginsubsection{Decision Trees Explaining Hard Clustering}\label{sec:decision-trees-explaining-hard-clustering}
A significant body of work on explainable (hard) clustering builds upon unsupervised decision trees \cite{bandyapadhyayHowFindGood2023}. Generally, clustering explanation techniques draw inspiration from supervised classification approaches. A common strategy is to treat the output of a hard clustering algorithm as ground truth and then fit a decision tree classifier to these results, as illustrated in Figure \ref{fig:scheme-imm}.

\begin{figure}[ht]
    \centering
    \resizebox{0.8\linewidth}{!}{\begin{tikzpicture}
    \node[] (input) at (0,0) {$X$};
    \node [draw, rectangle, 
    minimum height=3.0em, minimum width=8em, right = 1 of input, name=clustering]{hard clustering};
    \node[draw, rectangle, 
    minimum height=3.0em, minimum width=8em, right = 2.5cm of clustering, name=dt]{DT Training};
    \node[right = 1cm of dt] (output) {$\Delta :\mathbb{X} \rightarrow \Omega$};

    \draw [->, very thick] (input) -- (clustering);
    \draw [->, very thick] (clustering) -- (dt) node[midway, above] {$\mathcal{C} : X \rightarrow \Omega$};
    \draw [->, very thick] (dt) -- (output);
\end{tikzpicture}}
    \caption{Scheme of explainable clustering.}
    \label{fig:scheme-imm}
    \vspace{-10pt}
\end{figure}
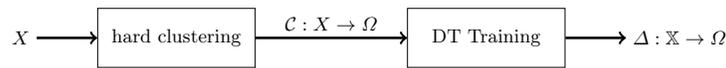

It is then natural to use $\operatorname{Accuracy}_\mathcal{C}(\Delta)$ as the quality measure for how well the decision tree $\Delta$ explains the original clustering $\mathcal{C}$. A notable algorithm developed for this purpose is the Iterative Mistake Minimization (IMM) \cite{moshkovitzExplainableKMeansKMedians2020}. IMM is a decision tree fitting algorithm that leverages the original hard cluster centroid structure to construct an explainer where each leaf contains exactly one centroid, mapping the points in the leaf to the cluster associated with that centroid. This approach relies on the concept of a mistake \cite{moshkovitzExplainableKMeansKMedians2020}.

\begin{definition} \label{def:mistake}
    Let $\mathcal{C}$ be a hard clustering. A \textbf{mistake} in a DT node $S \subseteq \mathbb{X}$ occurs when a point $x\in S$ has its associated cluster centroid $v_{\mathcal{C}(x)}$ outside of $S$, i.e., $v_{\mathcal{C}(x)} \notin S$.
\end{definition}  

The \textbf{number of mistakes} in a decision tree is the sum of mistakes committed with respect to $\mathcal{C}$ across all leaves. The IMM aims to minimize the number of mistakes induced by each split in the decision tree. At each step of the iterative process, the number of mistakes quantifies the explanation cost, that is, the accuracy loss introduced when using the decision tree explainer compared to the original clustering.

\section{An Algorithm for Explaining Evidential Clustering}\label{sec:algorithm}

Our objective is to extend the concept of decision trees for cluster explanation to the evidential setting. We aim to construct a decision tree that provides a representative approximation of the evidential clustering function, as illustrated in Figure \ref{fig:scheme-iemm}.

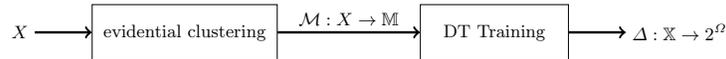
\begin{figure}[ht]
    \centering
    \resizebox{0.8\linewidth}{!}{\begin{tikzpicture}
    \node[] (input) at (0,0) {$X$};
    \node [draw, rectangle, 
    minimum height=3.0em, minimum width=10em, right = 1 of input, name=clustering]{evidential clustering};
    \node[draw, rectangle, 
    minimum height=3.0em, minimum width=8em, right = 2.5cm of clustering, name=dt]{DT Training};
    \node[right = 1cm of dt] (output) {$\Delta:\mathbb{X} \rightarrow 2^\Omega$};

    \draw [->, very thick] (input) -- (clustering);
    \draw [->, very thick] (clustering) -- (dt) node[midway, above] {$\mathcal{M} : X \rightarrow \mathbb{M}$};
    \draw [->, very thick] (dt) -- (output);
\end{tikzpicture}}
    \caption{Scheme of Explainable Evidential Clustering.}
    \label{fig:scheme-iemm}
    \vspace{-10pt}
\end{figure}

To achieve this, and inspired by IMM, we generalize the notion of a mistake for evidential classifiers and propose an algorithm that seeks to minimize this generalized loss.

\subsection{Explaining Classifiers under Uncertainty and Imprecision}

Let us consider a categorical evidential classifier\footnote{We use the term classifier to emphasize that the development of this section is valid not only for clustering but also for all (categorical) evidential partitions of the data.} $\mathcal{M}_c : X \rightarrow \mathbb{M}$. We define $\overline{\mathcal{M}}_c : X \rightarrow 2^\Omega$ such that, for any $x \in X$ and $A \subseteq \Omega$, $\mathcal{M}_c(x)(A) = 1$ if and only if $\overline{\mathcal{M}}_c(x) = A$.

\begin{definition}
    A \textbf{cautious explainer} for the categorical evidential classifier $\mathcal{M}_c$ is a map $\chi_{\mathcal{M}_c} : 2^\Omega \rightarrow 2^\mathbf{C}$.
\end{definition}

\revmarginsubsection{Utility}
While in the hard case the DT explainer quality relates closely to representativity (details in appendix \ref{app:dts-as-explainers}), in the evidential setting this concept needs refinement since errors vary in severity. For instance, assigning an observation $x$ with $\overline{\mathcal{M}}_c(x) = \{\omega_1\}$ to $\{\omega_1, \omega_2\}$ is potentially less severe than assigning it to $\{\omega_2\}$ alone. To capture these nuances, we introduce the concept of utility that quantifies the satisfaction of a decision-maker when assigning a subset $A$ when the ground truth is $B$.

\begin{definition}
    A \textbf{utility function} is a map $\mathcal{U} : 2^\Omega \times 2^\Omega \rightarrow [0,1]$ such that, $\forall A, B \in 2^\Omega$, $\mathcal{U}(A, A) = 1$ and $A \cap B = \emptyset \Rightarrow \mathcal{U}(A, B) = 0$.
\end{definition}

\revmarginsubsection{Costs}
In the context of a categorical evidential partition, we want to characterize the cost of explaining a point $x$ with a cautious explainer induced by some interpretable classifier $\Delta: X \rightarrow 2^\Omega$.

The utility $\mathcal{U}(A, \overline{\mathcal{M}}_c(x))$ quantifies the satisfaction of assigning metacluster $A$ to observation $x$ and, therefore, equals the cost of not assigning $A$ to $x$. Thus, the total cost for $x$ can be understood as $\overline{\operatorname{Cost}_{\mathcal{M}_c, \Delta}} : X \rightarrow [0,|\mathbb{F}|-1]$, the sum of costs from not assigning $x$ to all metaclusters $A \neq \Delta(x)$, 
\begin{equation}
    \overline{\operatorname{Cost}_{\mathcal{M}_c, \Delta}}(x) = \sum_{A \neq \Delta(x)} \mathcal{U}(A, \overline{\mathcal{M}}_c(x)).
    \label{eq:cost_up}
\end{equation}

Conversely, $1-\mathcal{U}(A, \overline{\mathcal{M}}_c(x))$ represents the cost of assigning $A$ to $x$, leading to $\underline{\operatorname{Cost}_{\mathcal{M}_c, \Delta}} : X \rightarrow [0,1]$, an alternative expression for total cost:
\begin{equation}
    \underline{\operatorname{Cost}_{\mathcal{M}_c, \Delta}}(x) = 1-\mathcal{U}(\Delta(x), \overline{\mathcal{M}}_c(x)).
    \label{eq:cost_down}
\end{equation}

It always holds that $\underline{\operatorname{Cost}_{\mathcal{M}_c, \Delta}}(x) \leq \overline{\operatorname{Cost}_{\mathcal{M}_c, \Delta}}(x)$. When $\mathcal{U}(A, B) = \mathbbm{1}_{A = B}$, the equality holds. Furthermore, if $\mathcal{M}_c$ induces a hard clustering and $\Delta$ is IMM-like (with exactly one centroid per leaf), both equal one (and not zero) if and only if $x$ is a mistake as stated in Definition \ref{def:mistake}.

\marginsubsection{A simple example}
Figure \ref{fig:utility} and Table \ref{tab:utility_costs} illustrate the utility concept through a concrete example. The value $\mathcal{U}(\{\omega_1, \omega_2\}, \{\omega_1\})$ quantifies how tolerable the mistake at $x_1$ is. When $\mathcal{U}(\{\omega_1, \omega_2\}, \{\omega_1\}) = 0$, this mistake becomes as intolerable as the one at $x_2$. Conversely, when $\mathcal{U}(\{\omega_1, \omega_2\}, \{\omega_1\}) = 1$, it becomes as tolerable as the correct assignment at $x_0$. In this latter scenario, removing $x_2$ from the dataset would yield an optimal assignment.

\begin{table}[h]
  \vspace{-10pt}
  \begin{minipage}[c]{0.48\linewidth}
    \centering
        \includegraphics[width=\linewidth]{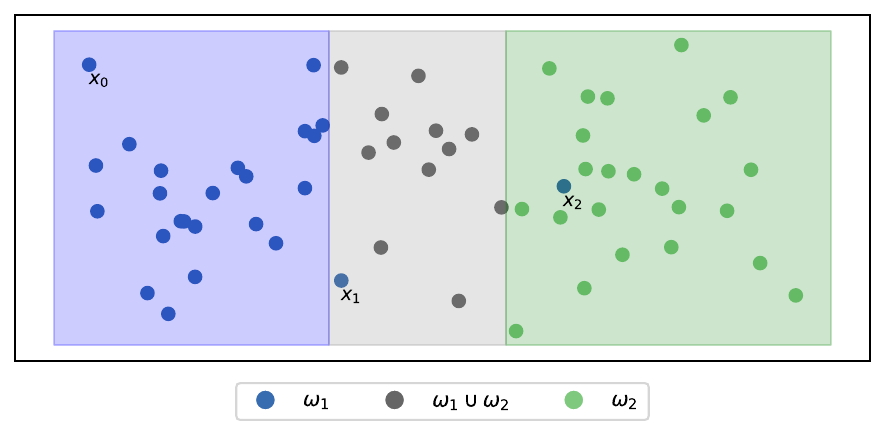}
        \captionof{figure}{Illustration of a categorical evidential classifier and space partition in $\mathbb{X} = \mathbb{R}^2$. The partition $\Delta$ separates $x_1$ and $x_2$ from their respective metaclusters, while correctly assigning all other observations.}
    \label{fig:utility}
  \end{minipage}
  \hfill
  \begin{varwidth}[c]{0.5\linewidth}
    \centering
    \resizebox{\linewidth}{!}{%
    {\renewcommand{\arraystretch}{1.5} 
    \begin{tabular}{c c c c}
        \toprule
        $x$ & $x_0$ & $x_1$ & $x_2$ \\
        \midrule
        $\Delta(x) $ & $\{\omega_1\}$ & $\{\omega_1, \omega_2\}$ & $\{\omega_2\}$ \\
        $\underline{\operatorname{Cost}_{\mathcal{M}_c, \Delta}}(x)$ & $0$ & $1 - \mathcal{U}(\{\omega_1, \omega_2\}, \{\omega_1\})$ & $1$ \\
        $2^\Omega \setminus \Delta(x) $ & $\{ \{\omega_2\}, \{\omega_1, \omega_2\}\}$ & $\{ \{\omega_1\}, \{\omega_2\}\}$ & $\{ \{\omega_1\}, \{\omega_1, \omega_2\}\}$ \\
        $\overline{\operatorname{Cost}_{\mathcal{M}_c, \Delta}}(x)$ & $\mathcal{U}(\{\omega_1, \omega_2\}, \{\omega_1\})$ & $1$ & $1 + \mathcal{U}(\{\omega_1, \omega_2\}, \{\omega_1\})$ \\
        \bottomrule
    \end{tabular}
    }}
    \vspace{5pt}
    \caption{For each point highlighted in Figure \ref{fig:utility}, we present the point $x$, the metacluster assigned by classifier $\Delta$, the cost of assigning $x$ to the metacluster designated by $\Delta$, the set of metaclusters not assigned by $\Delta$, and the cost of not assigning $x$ to them. Note that $\overline{\mathcal{M}}_c(x_0) = \overline{\mathcal{M}}_c(x_1) = \overline{\mathcal{M}}_c(x_2) = \{\omega_1\}$.}
    \label{tab:utility_costs}
  \end{varwidth}%
\vspace{-20pt}
\end{table}

\marginsubsection{Evidential Representativeness}
In the cautious case, a representative explainer is one that minimizes the total cost, assigning each observation to a metacluster that maximizes the utility.

\begin{definition}
A \textbf{$\mathcal{U}$-representative cautious explainer} for a categorical evidential clustering is a cautious explainer $\chi_{\mathcal{M}_c}$ such that, $\forall A \in 2^\Omega$, $\forall x \in \overline{\mathcal{M}}_c^{-1}(\{A\})$, there exists $L \in \bigcup_{\mathcal{U}(B, A) = 1} \chi_{\mathcal{M}_c}(B)$ such that, $\forall \langle\mathcal{A}, v \rangle \in L$, $x_\mathcal{A} = v$.
\end{definition}

The choice of utility function crucially determines what makes an explainer representative. Intuitively, utility functions yielding higher values are more "flexible" and tolerate more mistakes. While accuracy in the hard case measured explanation quality (through the representativity of the decision tree explainer), in the categorical evidential case, utility allows us to define the $\mathcal{U}$\textbf{-categorical representativeness} of a cautious explainer as its mean accuracy:
\begin{equation}
    \mathcal{R}_{\mathcal{M}_c, \mathcal{U}}(\Delta) = \frac{1}{|X|} \sum_{x \in X} \mathcal{U}(\Delta(x), \overline{\mathcal{M}}_c(x)).
    \label{eq:precision_categorical}
\end{equation}

For any evidential partition $\mathcal{M} : X \rightarrow \mathbb{M}$, we can extend this. Let the \textbf{$\mathcal{U}$-evidential representativeness} $\mathcal{R}_{\mathcal{M}, \mathcal{U}} : (\mathbb{X} \rightarrow 2^\Omega) \rightarrow [0,1]$ of a cautious explainer be the expected categorical representativeness weighted by the mass function:
\begin{equation}
    \mathcal{R}_{\mathcal{M}, \mathcal{U}}(\Delta) = \frac{1}{|X|} \sum_{x \in X} \sum_{B \in \mathbb{F}_{\mathcal{M}}} \mathcal{U}(\Delta(x), B) m_x(B).
    \label{eq:precision_evidential}
\end{equation}

Equation \eqref{eq:precision_categorical} is clearly a special case of \eqref{eq:precision_evidential} when $\mathcal{M}$ is categorical. Additionally, equations \eqref{eq:precision_evidential} and \eqref{eq:precision_hard} coincide when $\mathcal{M}$ is a hard clustering and $\mathcal{U}(A, B) = \mathbbm{1}_{A = B}$.

\revmarginsubsection{Evidential Mistakeness}
With our updated representativeness notion, we can now parallel the concept of mistake for the evidential case as a cost function capturing the representativeness loss associated with a single DT explainer node. Recall from Definition \ref{def:mistake} that, in the hard case, the number of mistakes in a node $S$ can be described in two equivalent ways:
\begin{enumerate}
    \item The number of mistakes in $S$ is the number of points $x \in S$ such that $\exists v_\omega \notin S$ with $\omega = \mathcal{C}(x)$.
    \item The number of mistakes in $S$ is the number of points $x \in S$ such that $\forall v_\omega \in S$, $\omega \neq \mathcal{C}(x)$.
\end{enumerate}

The first formulation relates to cost given at equation \eqref{eq:cost_up}, while the second relates to equation \eqref{eq:cost_down}. Translating these to the evidential setting yields two natural definitions of \textbf{evidential mistakeness}:
\begin{enumerate}
    \item The evidential mistakeness in $S$ is the sum of the costs introduced by not assigning points $x$ in $S$ to metaclusters that are not in $S$:
    \begin{equation}
        \overline{M}_{\mathcal{M}, \mathcal{U}}(S) = \sum_{x \in S} \sum_{v_A \notin S} \sum_{B \in \mathbb{F}_{\mathcal{M}}} \mathcal{U}(A, B)m_x(B).
        \label{eq:mistakeness_up}
    \end{equation}
    \item The evidential mistakeness in $S$ is the sum, over all $x$ in $S$, of the expected cost of assigning $x$ to some metacluster in $S$:
    \begin{equation}
        \underline{M}_{\mathcal{M}, \mathcal{U}}(S) = \sum_{x \in S} \sum_{v_A \in S} \sum_{B \in \mathbb{F}_{\mathcal{M}}} \frac{(1 - \mathcal{U}(A, B)) m_x(B)}{|\{A \in \mathbb{F}_{\mathcal{M}} : v_A \in S\}|}
        \label{eq:mistakeness_down}
    \end{equation}
\end{enumerate}

When $\mathcal{M}$ induces a hard clustering and $\mathcal{U}(A, B) = \mathbbm{1}_{A = B}$, Equation \eqref{eq:mistakeness_up} equals the number of mistakes in $S$. Additionally, the total cost of a cautious DT explainer induced by Equation \eqref{eq:mistakeness_down} is zero if and only if the explainer is $\mathcal{U}$-representative.

For IMM-like algorithms where all leaves $S$ contain exactly one centroid, both evidential mistakeness forms from Equations \eqref{eq:mistakeness_up} and \eqref{eq:mistakeness_down} are minimized by explanations with maximal evidential representativeness (see proof on appendix \ref{app:mistakeness-precision}). Furthermore, for IMM-like algorithms, if $\mathcal{M}$ is categorical:
$$\overline{M}_{\mathcal{M}, \mathcal{U}}(S) = \sum_{x \in S} \overline{\operatorname{Cost}_{\mathcal{M}, \Delta}}(x) \text{ and } \underline{M}_{\mathcal{M}, \mathcal{U}}(S) = \sum_{x \in S} \underline{\operatorname{Cost}_{\mathcal{M}, \Delta}}(x).$$

The key difference between these definitions emerges in nodes containing multiple centroids, where Equation \eqref{eq:mistakeness_down} penalizes such nodes more heavily than Equation \eqref{eq:mistakeness_up}. This makes Equation \eqref{eq:mistakeness_up} better suited for conservative explainers, while Equation \eqref{eq:mistakeness_down} is preferable for more optimistic ones.

\subsection{Implementation}

\marginsubsection{Choosing a Utility Function}
When an explainer yields a metacluster $A$, while the original classifier assigns $B$, two types of errors can occur. The first is insufficient coverage, measured by $|A^C \cap B|$ - where the explainer fails to include all elements of the true metacluster. The second is excessive coverage, measured by $|A \cap B^C|$ - where the explainer includes elements not in the true metacluster. Penalizing insufficient coverage indicates the explainer is not cautious enough, while penalizing excessive coverage suggests it is too cautious.

To address both error types, we introduce two families of utility functions with a positive parameter $\lambda$ controlling their behavior:
$$
    \overline{\mathcal{U}}^\lambda(A, B) = \left( \frac{|A \cap B|}{|A \cup B|} \mathbbm{1}_{B \subset A} \right)^{1/\lambda}   
\text{ and }
    \underline{\mathcal{U}}^\lambda(A, B) = \left( \frac{|A \cap B|}{|A \cup B|} \mathbbm{1}_{A \subset B} \right)^{1/\lambda}.
$$

These utility functions exhibit complementary tolerance behaviors. The function $\overline{\mathcal{U}}^\lambda(A, B)$ assigns zero utility when $A^C \cap B \neq \emptyset$, making it completely intolerant to insufficient coverage while allowing parameter $\lambda$ to modulate tolerance to excessive coverage. Higher $\lambda$ values reduce penalties for excessive coverage, embodying a more cautious approach. In contrast, $\underline{\mathcal{U}}^\lambda(A, B)$ assigns zero utility when $A \cap B^C \neq \emptyset$, showing complete intolerance to excessive coverage while $\lambda$ controls the degree of tolerance to insufficient coverage—higher $\lambda$ values reducing penalties for insufficient coverage and representing a more optimistic approach.

These combine into a comprehensive family of utility functions for $\lambda \in \mathbb{R}^*$:
$$\mathcal{U}^\lambda(A, B) = \begin{cases}
    \underline{\mathcal{U}}^{|\lambda|}(A, B) & \text{if } \lambda < 0 \\
    \overline{\mathcal{U}}^{|\lambda|}(A, B) & \text{if } \lambda > 0
\end{cases}.$$

We also define special cases as limits when $\lambda$ approaches 0 and $\pm \infty$. That gives $\mathcal{U}^0(A, B) = \mathbbm{1}_{A = B}, \
\mathcal{U}^{-\infty}(A, B) = \mathbbm{1}_{A \subset B} \text{ and }
\mathcal{U}^\infty(A, B) = \mathbbm{1}_{B \subset A}.$

Finally, based on these, we define the $\lambda$\textbf{-evidential mistakeness} as:
\begin{equation}
    M^\lambda_\mathcal{M} = 
    \begin{cases}
        \overline{M}_{\mathcal{M}, \mathcal{U}^\lambda} & \text{if } \lambda \geq 0 \\
        \underline{M}_{\mathcal{M}, \mathcal{U}^\lambda} & \text{if } \lambda < 0
    \end{cases}
\end{equation}
for any $\lambda \in \mathbb{R} \cup \{\pm\infty\}$. The higher the $\lambda$, the more the mistakeness function represents a conservative approach.


\revmarginsubsection{The Algorithm}
Inspired by IMM, we propose the \textbf{Iterative Evidential Mistake Minimization (IEMM)}.
\begin{wrapfigure}{L}{0.67\textwidth}
\begin{minipage}{0.67\textwidth}
    \vspace{-40pt}
    \input{figs/algos/iemm.tex} 
    \vspace{-40pt}
\end{minipage}
\end{wrapfigure}
The IEMM fits a decision tree based on an evidential clustering by minimizing the evidential mistakeness function (see Algorithm \ref{algo:iemm}). Each iteration of IEMM, for a region $S \subseteq \mathbb{X}$, considers a subset $F \subseteq \mathbb{F}_{\mathcal{M}}$ of the focal sets whose metacluster centroids lie within $S$ and finds the split that, by separating centroids, minimizes the contribution to the evidential mistakeness.

This algorithm extends IMM by accepting evidential partitions as input. When the underlying clustering function is hard and $\lambda=0$ is chosen, the algorithm operates identically to IMM because evidential mistakeness equals the number of mistakes. For cautious partitions as input, varying $\lambda$ controls the "level of cautiousness" of the explainer.

From a computational perspective, the baseline IMM algorithm has complexity $O(C \cdot D \cdot N \cdot \log N)$ \cite{moshkovitzExplainableKMeansKMedians2020}, where $C = |\Omega|$ represents the number of clusters, $D$ the dimensionality of the feature space, and $N$ the sample count. Our IEMM algorithm extends this by incorporating utility computations at each node, adding an $O(K^2)$ factor where $K = |\mathbb{F}|$ is the number of metaclusters. This results in a total complexity of $O(K^2 \cdot D \cdot N \cdot \log N)$. In practice, explainable decision trees typically employ a modest number of metaclusters, mitigating potential performance concerns.

We have implemented IEMM using Python 3.9.6. All code is available at \url{https://github.com/victorsouza89/iemm}. The implementation of a decision tree accepting evidential labels was based on the code made available by \cite{hoarauEvidentialRandomForests2023}.

\marginsubsection{The Tests}
Using the \texttt{evclust} library \cite{soubeigaEvclustPythonLibrary2025}, we generated three evidential partitions over synthetic datasets of 2 features ($x$ and $y$). Those were:
\begin{itemize}
    \item A dataset of 200 entries over which we defined $\mathcal{M}_\text{easy}$, with $\Omega = \{\omega_1, \omega_2\}$ and $\mathbb{F}_{\mathcal{M}_\text{easy}} = 2^{\Omega} \setminus \emptyset$.

    \item A dataset of 300 samples and, for $\Omega = \{\omega_1, \omega_2, \omega_3\}$, we generated two types of evidential clustering functions:
\begin{itemize}
    \item $\mathcal{M}_\text{full}$, an evidential clustering with $\mathbb{F}_{\mathcal{M}_\text{full}} = 2^{\Omega} \setminus \emptyset$.
    
    \item $\mathcal{M}_\text{quasi-bayesian}$, an evidential clustering that is a quasi-bayesian clustering function. This means that the focal sets are the singletons and the whole space. That is, $\mathbb{F}_{\mathcal{M}_\text{quasi-bayesian}} = \{\{\omega_1\}, \{\omega_2\}, \{\omega_3\}, \{\omega_1, \omega_2, \omega_3\}\}$.
\end{itemize}
\end{itemize}

\begin{wrapfigure}{r}{0.6\textwidth}
    \centering
    \includegraphics[width=\linewidth]{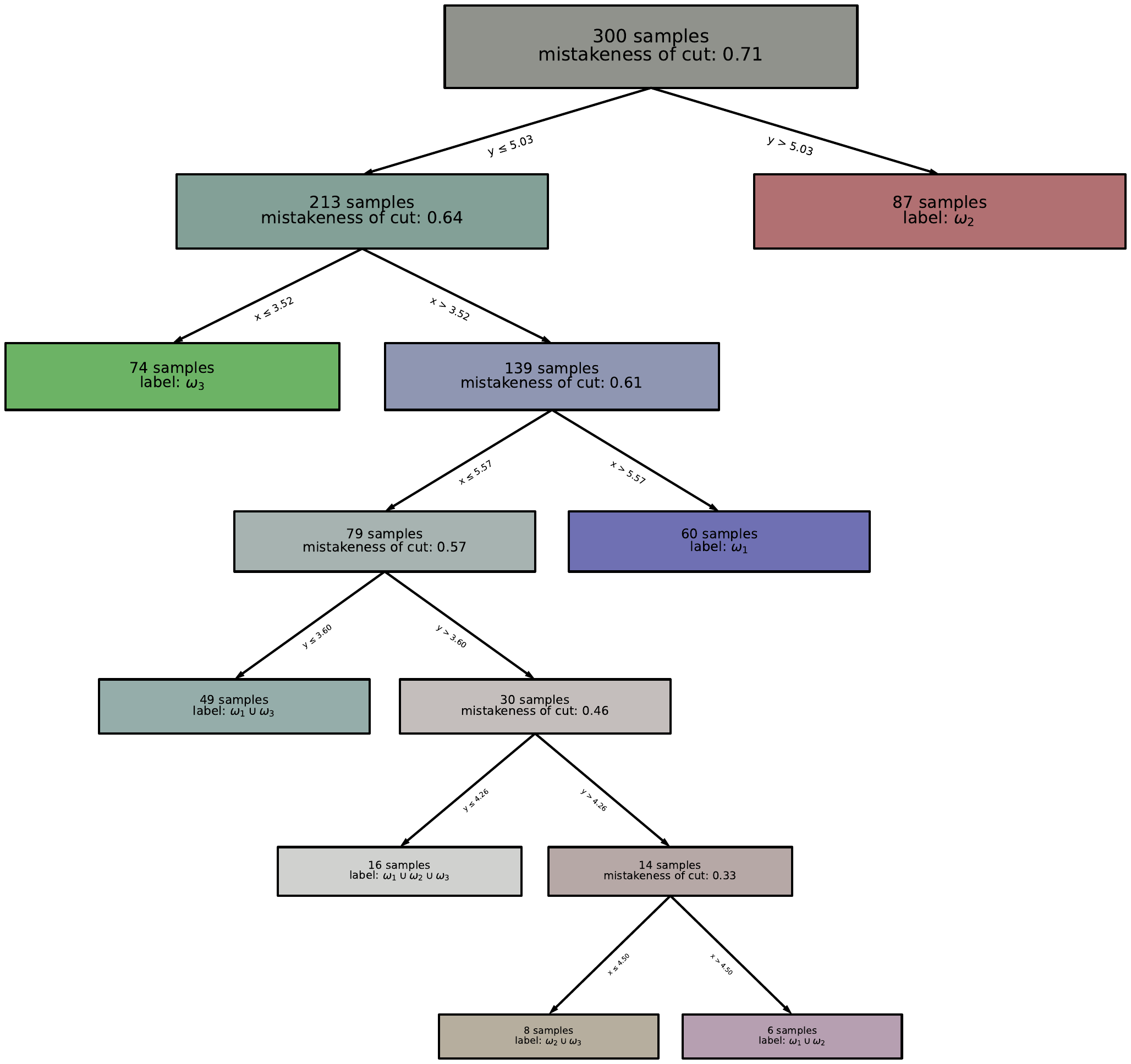}
    \caption{Decision tree obtained with IEMM for the evidential clustering function $\mathcal{M}_\text{full}$ and $\lambda=0$. The decision trees generated by IEMM are shallow by construction, having at most $|\mathbb{F}|-1$ levels. Each non-terminal node indicates the mistakeness of the corresponding split.}
    \label{fig:tree-plot}
    \vspace{-20pt}
\end{wrapfigure}

Then, for each evidential clustering function, we constructed a decision tree using the IEMM and the $\lambda$-evidential mistakeness function for different values of $\lambda$. The partition of the space induced by these explanations is illustrated in Figure \ref{fig:plots}. The decision tree explainer for $\mathcal{M}_\text{full}$ and $\lambda=0$ is represented in Figure \ref{fig:tree-plot}. The resulting explanations for the quasi-bayesian clustering function are shown in Table \ref{tab:paths_qb}, appendix \ref{app:results}.  

Table \ref{tab:results_synthetic} (appendix \ref{app:results}) presents the representativeness achieved in each scenario. 
Notably, the decision tree generated by fixing $\lambda=\infty$ over the $\mathcal{M}_\text{easy}$ dataset achieves the highest representativeness observed, surpassing 93\%. This can be interpreted as the expected explanation accuracy.
To further validate our approach, we evaluated it on larger real-world datasets from the \texttt{credal-datasets-master} repository \cite{hoarauDatasetsRichLabels2023}, with results shown in Table \ref{tab:results} (appendix \ref{app:results}).

Across both synthetic and real-world datasets, decision trees obtained using the $\lambda$-evidential mistakeness function consistently achieve the best performance in terms of $\mathcal{U}^\lambda$-evidential representativeness. The gap between the $\mathcal{U}^\lambda$-evidential representativeness and 1 quantifies the loss of accuracy or cost of the explanation.

\begin{figure*}[t!]
    \centering
    \includegraphics[width=\linewidth]{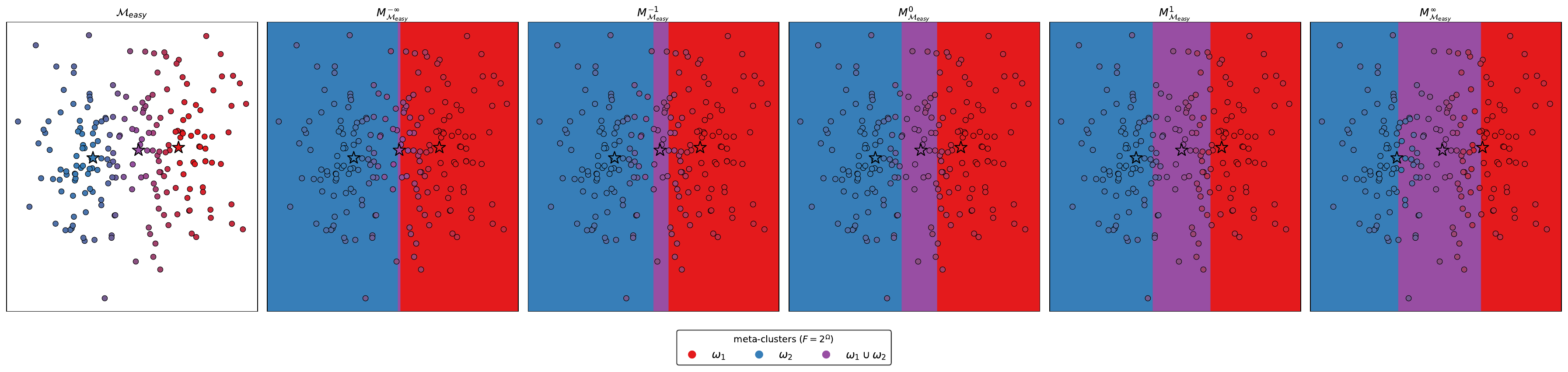}
    \includegraphics[width=\linewidth]{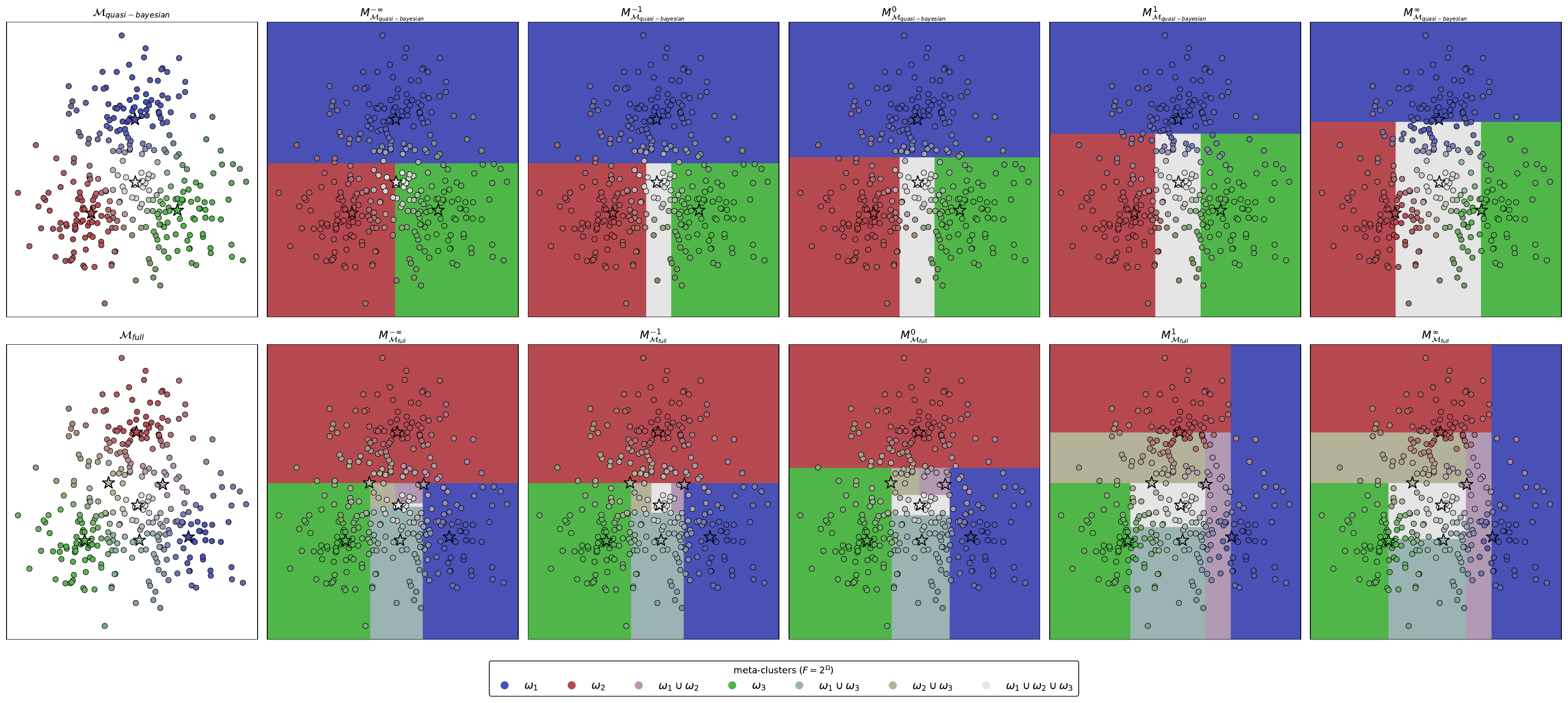}
    \caption{The results of the IEMM on the synthetic dataset for the evidential clustering functions $\mathcal{M}_\text{easy}$, $\mathcal{M}_\text{full}$, $\mathcal{M}_\text{quasi-bayesian}$ and different utility functions. The star represents the centroid of each metacluster. The utility function strongly influences the resulting explanations. The higher the $\lambda$, the more the $\lambda$-evidential mistakeness function assigns larger portions of the space to metaclusters representing doubt. At the limit, $\lambda = -\infty$ (column 2), the obtained explanations give the maximum possible space to the singleton metaclusters. Conversely, $\lambda = \infty$ (column 6) assigns the maximum possible space to the metaclusters representing doubt.}
    \label{fig:plots}
    \vspace{-10pt}
\end{figure*}

\section{Conclusion}\label{sec:conclusion}
In this paper, we presented a novel approach to explainable evidential clustering using decision trees as explainers. We established that representativity constitutes both a necessary and sufficient condition for decision trees to function as abductive explainers, and extended this concept to imprecise settings through the introduction of a utility function. This utility function allows for the accommodation of "tolerable" mistakes in explanations, making it particularly suitable for evidential contexts. Building on these theoretical foundations, we proposed the evidential mistakeness measure and developed the Iterative Evidential Mistake Minimization (IEMM) algorithm. Our approach produces decision trees that effectively explain evidential clustering functions, advancing the development of both cautious and explainable AI systems.

An important consideration regards the expected audience of the explanations our algorithm creates. Our work implicitly assumes that, in the evidential theory framework, decision-makers possess domain expertise, an understanding of the implications of their choices, and knowledge about their risk tolerance preferences. The explanations we generate are designed for these informed stakeholders—individuals familiar with the feature space and its relationships. For example, in clinical applications, our explanations target medical professionals who can appropriately interpret physiological measurements, rather than patients without specialized knowledge.

A notable property of IEMM is the generation of inherently shallow decision trees. Following the IMM design principle, IEMM produces exactly one leaf per cluster, limiting tree depth to at most $|\mathbb{F}|-1$. This structural constraint enhances interpretability—a primary goal of explainable AI—though it may occasionally result in explanations that cannot fully capture complex data patterns, potentially creating overly rigid explainers for certain applications.

Our research opens several promising avenues for future investigation, particularly in two key directions:
\begin{itemize}
    \item \textbf{Domain-specific utility formulations}: While we have proposed a family of natural constructions for utility functions, domain-specific adaptations warrant further exploration. Future research could investigate more complex utility functions or models in which disjoint metaclusters are not treated as completely incompatible, thereby better reflecting the nuanced uncertainty relationships in specialized domains.
    
    \item \textbf{Advanced interpretable evidential classifiers}: Developing more sophisticated interpretable evidential classifiers that maintain or exceed the performance of standard decision trees represents a significant opportunity. Potential approaches include soft decision trees \cite{olaruCompleteFuzzyDecision2003}, fuzzy rule-based learning systems \cite{nozakiAdaptiveFuzzyRulebased1996}, and linguistic \cite{zadehFuzzySets1965} cautious explainers. Additionally, extending these methods to better account for outliers could enhance their robustness in real-world applications.
\end{itemize}

In conclusion, by advancing methods for cautious and explainable clustering, our work contributes to the broader goal of developing AI systems that effectively handle uncertainty while remaining interpretable to human experts. The IEMM algorithm and its theoretical foundations represent a step toward AI systems that acknowledge imperfect information, incorporate domain expertise, and communicate their reasoning in an accessible manner—all key requirements for responsible AI deployment in high-stakes decision-making contexts.

%
\newpage
\bibliographystyle{splncs04}
\bibliography{bib}

\newpage
\appendix
\section{Appendix: On Simplification Explanation Techniques} \label{app:simplification-explanation-techniques}

Simplification techniques typically rely on rule extraction methods, encompassing both global and local approaches. Studies have been conducted to assess the quality of these explanations \cite{amgoudAxiomaticFoundationsExplainability2022}. Below, we introduce definitions that characterize effective explanations produced by simplification techniques.

A \textbf{feature literal} is a pair $\langle\mathcal{A}, v \rangle$ where $\mathcal{A} \in \mathcal{D}$ and $v \in \mathcal{A}$. Let $\mathbf{L}$ be the set of all feature literals for all attributtes. A consistent subset of feature literals is $L \subset \mathbf{L}$ such that $\langle\mathcal{A}, v \rangle, \langle\mathcal{A}, v' \rangle \in L \Rightarrow v = v'$. Let $\mathbf{C} \subseteq 2^{\mathbf{L}}$ be the set of \textbf{all consistent subsets of feature literals}. Each $\mathbf{D} \subset \mathbf{C}$ induces a map $\operatorname{DNF} : \mathbb{X} \rightarrow \{\text{True}, \text{False}\}$ with
\begin{equation}
\operatorname{DNF}_\mathbf{D}(x) = \bigvee_{L \in \mathbf{D}} \left( \bigwedge_{\langle\mathcal{A}, v \rangle \in L}\left(x_\mathcal{A}=v\right)\right)
\label{eq:DNF}
\end{equation}
which is a Disjunctive Normal Form (DNF) \cite{suInterpretableTwolevelBoolean2015}. A DNF can serve as a human-interpretable classification model. When a DNF matches the behavior of a black-box classifier, we achieve a particularly desirable outcome known as an abductive explanation.

\noindent\fbox{%
    \parbox{\textwidth}{%
        \textbf{A Concrete Example:} Consider a philosopher studying living beings who observes two key characteristics: their appearance and their mode of locomotion. To formalize this classification problem, the philosopher defines the feature space of conceivable living beings as $\mathbb{X} = \texttt{App} \times \texttt{Move},$ where
        $$
        \texttt{App} = \{\texttt{feathered}, \texttt{featherless}\} \text{ and } \texttt{Move} = \{\texttt{biped}, \texttt{non-biped}\}.
        $$
        From these features, we can construct the set of feature literals:
        \begin{footnotesize}
        $$\mathbf{L} = \{\langle\texttt{App}, \texttt{feathered}\rangle, \langle\texttt{App}, \texttt{featherless}\rangle, \langle\texttt{Move}, \texttt{biped}\rangle, \langle\texttt{Move}, \texttt{non-biped}\rangle\}.$$ 
        \end{footnotesize}
        The set of all consistent subsets of feature literals encompasses all possible combinations that do not contain contradictory values for the same attribute:
        \begin{footnotesize}
        \begin{align*}
            &\mathbf{C} = \{\emptyset, \\ &\{\langle\texttt{App}, \texttt{feathered}\rangle\}, \{\langle\texttt{App}, \texttt{featherless}\rangle\}, \{\langle\texttt{Move}, \texttt{biped}\rangle\}, \{\langle\texttt{Move}, \texttt{non-biped}\rangle\}, 
            \\ &\{\langle\texttt{App}, \texttt{feathered}\rangle, \langle\texttt{Move}, \texttt{biped}\rangle\}, \{\langle\texttt{App}, \texttt{feathered}\rangle, \langle\texttt{Move}, \texttt{non-biped}\rangle\}, \\ &\{\langle\texttt{App}, \texttt{featherless}\rangle, \langle\texttt{Move}, \texttt{biped}\rangle\}, \{\langle\texttt{App}, \texttt{featherless}\rangle, \langle\texttt{Move}, \texttt{non-biped}\rangle\}\\ &\}.
        \end{align*}
        \end{footnotesize}

        An example of a DNF is given by $\mathbf{D} = \{\{\langle\texttt{App}, \texttt{featherless}\rangle, \langle\texttt{Move}, \texttt{biped}\rangle\}\}$, which corresponds to featherless bipedal beings. That is, for any living being $x$, we have $\operatorname{DNF}_\mathbf{D}(x) = (x_\texttt{App} = \texttt{featherless}) \land (x_\texttt{Move} = \texttt{biped})$. $\operatorname{DNF}_\mathbf{D}(x)$ is true whenever $x$ is a human being.

        Conversely, $\mathbf{D}' = \{\{\langle\texttt{App}, \texttt{feathered}\rangle\}, \{\langle\texttt{Move}, \texttt{non-biped}\rangle\}\}$ corresponds to beings that are either feathered or non-bipedal. In this case, the induced DNF is $\operatorname{DNF}_{\mathbf{D}'}(x) = (x_\texttt{App} = \texttt{feathered}) \lor (x_\texttt{Move} = \texttt{non-biped})$ and $\operatorname{DNF}_{\mathbf{D}'}(x)$ is false whenever $x$ is a human being.
    }%
}

\begin{definition}
  An \textbf{abductive explanation} of the label $\omega \in \Omega$ is a $L \in \mathbf{C}$ such that, $\forall x \in \mathbb{X}$,
  $$
  \left(\bigwedge_{\langle\mathcal{A}, v \rangle \in L}\left(x_\mathcal{A}=v\right)\right) \Rightarrow \Gamma(x) = \omega
  $$
\end{definition}

Abductive explanations were introduced to address the question: "Why is $\Gamma(x) = \omega$?", providing a sufficient reason for characterizing the label $\omega$, where $\Gamma$ is a supervised classifier \cite{ignatievAbductionBasedExplanationsMachine2019}. In the context of explainability, an ideal construction would be a system that can provide satisfactory explanations\footnote{In this work, we consider satisfactory explanations to be "abductive" or "as abductive as possible." However, this might not always be the case. As highlighted in multiple works \cite{barredoarrietaExplainableArtificialIntelligence2020}, the best type of explanation depends on the audience for which this explanation is intended. We develop this discussion further in the conclusion.} for a classifier's outputs.

\begin{definition}
  An \textbf{explainer} of a classifier $\Gamma : \mathbb{X} \rightarrow \Omega$ is a map $\chi_\Gamma : \Omega \rightarrow 2^\mathbf{C}$.
\end{definition}

That is, to each class $\omega$, a classifier associates a DNF. If the DNF issued from $\chi_\Gamma$ matches $\Gamma$, the classifier provides abductive explanations.

\newpage
\section{Appendix: Decision Trees as Explainers} \label{app:dts-as-explainers}
In this section, we provide a brief overview of decision trees (DTs) and their role as explainers. We also establish the relationship between representativity and abductivity in the context of DTs. We adapt the following definition of univariate decision trees from \cite{izzaTacklingExplanationRedundancy2022a}.

\begin{definition}
    The \textbf{graph of a decision tree} $\mathcal{T} = (V, E)$ is a directed acyclic graph in which there is at most one path between any two vertices. The vertex set $V$ is divided into non-terminal vertices $N$ and terminal vertices $T$, such that $V = N \cup T$. Additionally, $\mathcal{T}$ has a unique root vertex, $\operatorname{root}(\mathcal{T}) \in V$, which has no incoming edges, while every other vertex has exactly one incoming edge.

    To each graph of a DT, there are two important associated functions:
    \begin{itemize}
        \item A \textbf{split} is a map $\phi: N \rightarrow \mathcal{D}$ that assigns an attribute to each non-terminal vertex. 
        \item Let $\operatorname{children}(r) = \{s \in V \mid (r, s) \in E\}$ be the set of children of a vertex $r$. A \textbf{decision} is a map $\varepsilon: E \rightarrow \mathbf{L}$ such that, for every non-terminal vertex $r \in N$, there exists a bijection $\varepsilon_r : \operatorname{children}(r) \rightarrow \phi(r)$ satisfying $\varepsilon(r, s) = \langle \phi(r), \varepsilon_r(s) \rangle$.
    \end{itemize}
\end{definition}

It is well known that any binary decision tree can be transformed in linear time into an equivalent disjunctive normal form (DNF) expression \cite{audemardExplanatoryPowerBoolean2022a}. This property is often referenced when DTs are described as "interpretable." With this in mind, we associate each vertex with a path, which serves as the foundation for interpreting a decision tree as an explainer.

For a fixed graph of a DT $\mathcal{T}$, let $\operatorname{DNF}(r)$ be the set of literals associated with the edges that link the root to vertex $r$. All literals in $\operatorname{DNF}(r)$ are consistent \cite{izzaTacklingExplanationRedundancy2022a}. That is, $\operatorname{DNF} : V \rightarrow 2^\mathbf{C}$. Let $\mathbf{D} = \operatorname{DNF}(T)$ be the set of all DNFs associated with terminal vertices.

\begin{definition}
    A \textbf{path} is a map $\Upsilon:\mathbb{X} \rightarrow \mathbf{D}$ such that, for all $x \in \mathbb{X}$, 
    $$\bigwedge_{\langle\mathcal{A}, v \rangle \in \Upsilon(x)} (x_\mathcal{A} = v).$$
\end{definition}

The partitioning nature of decision trees ensures that each path is well-defined, meaning every possible observation follows a unique path. This characteristic allows us to interpret vertices as subsets of the feature space \cite{hoarauEvidentialRandomForests2023}.

\begin{definition}
    A \textbf{node} is a nonempty subset $S \subseteq \mathbb{X}$.
\end{definition}

Every achievable vertex can be trivially associated with a unique node by its DNF. We call \textbf{leaves} the nodes associated with terminal vertices. The set $\mathbf{D}$ can be understood as the explanation for each leaf. Associating leaves with explanations allows us to define the DT as a classifier.

\begin{definition} \label{def:dt}
    A \textbf{decision tree} is a map $\Delta:\mathbb{X} \rightarrow \Omega$ to which a path $\Upsilon_\Delta$ provides an abductive explanation. That is, $\forall x \in \mathbb{X}$, $$\bigwedge_{\langle\mathcal{A}, v \rangle \in \Upsilon(x)} (x_\mathcal{A} = v) \Rightarrow \Delta(x) = \omega.$$
\end{definition}

Let, $\forall \omega \in \Omega$, $\mathcal{L}_\omega^\Delta = \{\Upsilon^{-1}(\{L\}) : L \in \Upsilon(\Delta^{-1}(\{\omega\}))\}$ be the set of all leaves associated with the explanation of $\omega$. The \textbf{DT explainer} $\chi_\Gamma^\Delta$ associated with $\Delta$ is an explainer that, for any label, returns all paths explaining it. That is, $\chi_\Gamma^\Delta(\omega) = \Upsilon_\Delta [\mathcal{L}_\omega^\Delta]=\{\Upsilon_\Delta(x):x\in \mathcal{L}_\omega^\Delta\}$.

Our investigation focuses on the quality of explanations when, in the context of model simplification, the original classifier diverges from the decision tree explaining it. We borrow the concept of representative explainer from \cite{amgoudAxiomaticFoundationsExplainability2022}. A representative explainer is one that, for all observations $x$ with label $\omega = \Gamma(x)$, can provide an explanation $L$ that holds at $x$. That is, there exists a set of literals $L \in \chi_\Gamma(\omega)$ such that $\langle\mathcal{A}, v \rangle \in L \Rightarrow x_\mathcal{A} = v$.

\begin{definition}
  A \textbf{representative explainer} is an explainer $\chi_\Gamma$ such that, $\forall \omega \in \Omega, \ \forall x \in \Gamma^{-1}(\{\omega\})$, $\exists L \in \chi_\Gamma(\omega)$ such that, $\forall \langle\mathcal{A}, v \rangle \in L, x_\mathcal{A} = v$.
\end{definition}

The work in \cite{amgoudAxiomaticFoundationsExplainability2022} proves that every explainer providing abductive explanations is representative. We complement this result by proving that every representative DT explainer provides abductive explanations. Thus, for DT explainers, representativity and abductivity are equivalent.

\begin{theorem}
    If the $\chi_\Gamma^\Delta$ explainer is representative, it provides abductive explanations.
\end{theorem}
\begin{proof}
    We proceed by contradiction. 
    Assume the DT explainer does not provide abductive explanations. 
    
    From definition \ref{def:dt}, this implies that $\Gamma \neq \Delta$. That is, there exists $x \in \mathbb{X}$ such that $\omega_\Gamma = \Gamma(x) \neq \Delta(x) = \omega_\Delta$. Since the leaves form a partition of the feature space and $x \in \mathcal{L}_{\omega_\Delta}^{\Delta}$, we have $\Upsilon_\Delta(x) \notin \Upsilon_\Delta [\mathcal{L}_{\omega_\Gamma}^{\Delta}]$, and the explainer is not representative. \qed
\end{proof}

\newpage
\section{Appendix: Relating Mistakeness and Representativeness} \label{app:mistakeness-precision}

In this section, we show that the evidential representativeness and the total evidential mistakeness (the sum of the evidential mistakeness of each leaf) are equivalent in terms of measuring the quality of a IMM-like decision tree.

\begin{theorem}
    Let $\Delta, \Delta': \mathbb{X} \rightarrow 2^\Omega$ be two IMM-like decision trees. Then, for any evidential partition $\mathcal{M}$ and utility $\mathcal{U}$, 
\begin{align*}
    \mathcal{R}_{\mathcal{M}, \mathcal{U}}(\Delta) &\geq \mathcal{R}_{\mathcal{M}, \mathcal{U}}(\Delta') \\
    \iff \sum_{A \subset \Omega} \overline{M}_{\mathcal{M}, \mathcal{U}}(\mathcal{L}_A^\Delta) &\leq \sum_{A \subset \Omega} \overline{M}_{\mathcal{M}, \mathcal{U}}(\mathcal{L}_A^{\Delta'}) \\
    \iff \sum_{A \subset \Omega} \underline{M}_{\mathcal{M}, \mathcal{U}}(\mathcal{L}_A^\Delta) &\leq \sum_{A \subset \Omega} \underline{M}_{\mathcal{M}, \mathcal{U}}(\mathcal{L}_A^{\Delta'})
\end{align*}
    where $v_A \in \mathcal{L}_A^\Delta$ which is the leaf associated with the cluster $A$ in the decision tree $\Delta$.
\end{theorem}

\begin{proof}

    We start by establishing the relation between the two mistakenness functions. By definition, $x \in \mathcal{L}_A^\Delta \iff \Delta(x) = A$. From equations \eqref{eq:mistakeness_up} and \eqref{eq:mistakeness_down},
\begin{align*}
    \overline{M}_{\mathcal{M}, \mathcal{U}}(\mathcal{L}_A^\Delta) &= \sum_{x \in \mathcal{L}_A^\Delta} \sum_{\Delta(x) \neq C} \sum_{B \in \mathbb{F}_{\mathcal{M}}} \mathcal{U}(C, B)m_x(B), \\
    \underline{M}_{\mathcal{M}, \mathcal{U}}(\mathcal{L}_A^\Delta) &= \sum_{x \in \mathcal{L}_A^\Delta} \sum_{B \in \mathbb{F}_{\mathcal{M}}} (1 - \mathcal{U}(A, B)) m_x(B).
\end{align*} Then,
\begin{align*}
    & \sum_{A \subset \Omega} \underline{M}_{\mathcal{M}, \mathcal{U}}(\mathcal{L}_A^\Delta) - \sum_{A \subset \Omega} \overline{M}_{\mathcal{M}, \mathcal{U}}(\mathcal{L}_A^\Delta) \\
    &= \sum_{A \subset \Omega} \sum_{x \in \mathcal{L}_A^\Delta} \sum_{B \in \mathbb{F}_{\mathcal{M}}} m_x(B) \left((1 - \mathcal{U}(A, B)) - \sum_{A \neq C} \mathcal{U}(C, B)\right) \\
    &= \sum_{x \in X} \sum_{B \in \mathbb{F}_{\mathcal{M}}} m_x(B) \left(1 - \sum_{C \subset \Omega} \mathcal{U}(C, B)\right) \\
    &= |X| - \sum_{x \in X} \sum_{B \in \mathbb{F}_{\mathcal{M}}} m_x(B) \left( \sum_{C \subset \Omega} \mathcal{U}(C, B) \right) = |X| - \kappa_{\mathcal{M}, \mathcal{U}}.
\end{align*}
    where $\kappa_{\mathcal{M}, \mathcal{U}}$ is a constant that depends only on the evidential partition $\mathcal{M}$ and utility $\mathcal{U}$, but not on the specific decision tree $\Delta$.
    
    Also, from equation \eqref{eq:precision_evidential}, 
    $$
    |X| \mathcal{R}_{\mathcal{M}, \mathcal{U}}(\Delta) = \sum_{A \subset \Omega} \sum_{x \in \mathcal{L}_A^\Delta} \sum_{B \in \mathbb{F}_{\mathcal{M}}} \mathcal{U}(\Delta(x), B) m_x(B).
    $$
    Similarly,
\begin{align*}
    &|X| \mathcal{R}_{\mathcal{M}, \mathcal{U}}(\Delta) + \sum_{A \subset \Omega} \overline{M}_{\mathcal{M}, \mathcal{U}}(\mathcal{L}_A^\Delta) = \\
    &\sum_{A \subset \Omega} \sum_{x \in \mathcal{L}_A^\Delta} \sum_{B \in \mathbb{F}_{\mathcal{M}}} m_x(B) \left(\mathcal{U}(A, B) + \sum_{A \neq C} \mathcal{U}(C, B)\right) = \kappa_{\mathcal{M}, \mathcal{U}}.
\end{align*}
    
Since all three measures are related by affine transformations with the same constant terms, they preserve the same ordering relationships between different decision trees. Therefore, comparing two trees $\Delta$ and $\Delta'$ using any of these measures yields equivalent results, proving the stated equivalences. \qed
\end{proof}

\newpage
\section{Appendix: Results} \label{app:results}

In this section, we present the results of our experiments on synthetic and real-world datasets as described in section \ref{sec:algorithm}.

\begin{table}[h]
    \vspace{-10pt}
    \centering
    \resizebox{\linewidth}{!}{%
    \begin{tabular}{|c|c|c|c|c|}
\toprule
\hline
 & $\omega_{2}$ & $\omega_{1} \cup \omega_{2} \cup \omega_{3}$ & $\omega_{3}$ & $\omega_{1}$ \\
\midrule
\hline
$M^{-\infty}_{\mathcal{M}_{q-bay}}$ & $(y \leq 4.54) \wedge (x \leq 4.43)$ & $(y \leq 4.54) \wedge (x > 4.43) \wedge (x \leq 4.48)$ & $(y \leq 4.54) \wedge (x > 4.48)$ & $(y > 4.54)$ \\
$M^{-1}_{\mathcal{M}_{q-bay}}$ & $(y \leq 4.54) \wedge (x \leq 4.08)$ & $(y \leq 4.54) \wedge (x > 4.08) \wedge (x \leq 4.98)$ & $(y \leq 4.54) \wedge (x > 4.98)$ & $(y > 4.54)$ \\
$M^{0}_{\mathcal{M}_{q-bay}}$ & $(y \leq 4.69) \wedge (x \leq 3.85)$ & $(y \leq 4.69) \wedge (x > 3.85) \wedge (x \leq 5.09)$ & $(y \leq 4.69) \wedge (x > 5.09)$ & $(y > 4.69)$ \\
$M^{1}_{\mathcal{M}_{q-bay}}$ & $(y \leq 5.39) \wedge (x \leq 3.68)$ & $(y \leq 5.39) \wedge (x \leq 5.28) \wedge (x > 3.68)$ & $(y \leq 5.39) \wedge (x > 5.28)$ & $(y > 5.39)$ \\
$M^\infty_{\mathcal{M}_{q-bay}}$ & $(y \leq 5.82) \wedge (x \leq 2.95)$ & $(y \leq 5.82) \wedge (x \leq 5.95) \wedge (x > 2.95)$ & $(y \leq 5.82) \wedge (x > 5.95)$ & $(y > 5.82)$ \\
\hline
\bottomrule
\end{tabular}

    }
    \vspace{10pt}
    \caption{Abductive explanations generated by IEMM for all clusters of the quasi-bayesian evidential clustering function. Higher $\lambda$ values result in larger portions of the feature space being attributed to the cautious metacluster $\omega_1 \cup \omega_2 \cup \omega_3$.}
    \label{tab:paths_qb}
    \vspace{-30pt}
\end{table}

\begin{table}[h]
    \centering
    \resizebox{0.6\linewidth}{!}{%
    \begin{tabular}{|l|c|c|c|c|c|}
\toprule
\hline
 & $\mathcal{R}_{\mathcal{M}_{easy}, \mathcal{U}^{-\infty}}$ & $\mathcal{R}_{\mathcal{M}_{easy}, \mathcal{U}^{-1}}$ & $\mathcal{R}_{\mathcal{M}_{easy}, \mathcal{U}^{0}}$ & $\mathcal{R}_{\mathcal{M}_{easy}, \mathcal{U}^{1}}$ & $\mathcal{R}_{\mathcal{M}_{easy}, \mathcal{U}^{\infty}}$ \\
\hline
\midrule
$M^{-\infty}_{\mathcal{M}_{easy}}$ & \textbf{0.915796} & 0.808588 & 0.701380 & 0.701452 & 0.701524 \\
$M^{-1}_{\mathcal{M}_{easy}}$ & 0.901122 & \textbf{0.819012} & 0.736903 & 0.749377 & 0.761850 \\
$M^{0}_{\mathcal{M}_{easy}}$ & 0.876733 & 0.813867 & \textbf{0.751002} & 0.781731 & 0.812461 \\
$M^{1}_{\mathcal{M}_{easy}}$ & 0.781247 & 0.751198 & 0.721149 & \textbf{0.811562} & 0.901975 \\
$M^{\infty}_{\mathcal{M}_{easy}}$ & 0.689432 & 0.669249 & 0.649067 & 0.789613 & \textbf{0.930160} \\
\hline
\bottomrule
\end{tabular}

    }
    
    \resizebox{0.6\linewidth}{!}{%
    \begin{tabular}{|l|c|c|c|c|c|}
\toprule
\hline
 & $\mathcal{R}_{\mathcal{M}_{full}, \mathcal{U}^{-\infty}}$ & $\mathcal{R}_{\mathcal{M}_{full}, \mathcal{U}^{-1}}$ & $\mathcal{R}_{\mathcal{M}_{full}, \mathcal{U}^{0}}$ & $\mathcal{R}_{\mathcal{M}_{full}, \mathcal{U}^{1}}$ & $\mathcal{R}_{\mathcal{M}_{full}, \mathcal{U}^{\infty}}$ \\
\hline
\midrule
$M^{-\infty}_{\mathcal{M}_{full}}$ & \textbf{0.882009} & 0.731303 & 0.575128 & 0.593766 & 0.612354 \\
$M^{-1}_{\mathcal{M}_{full}}$ & 0.881689 & 0.738726 & 0.598038 & 0.618498 & 0.638218 \\
$M^{0}_{\mathcal{M}_{full}}$ & 0.867413 & \textbf{0.745508} & \textbf{0.625343} & 0.656206 & 0.683719 \\
$M^{1}_{\mathcal{M}_{full}}$ & 0.642447 & 0.596290 & 0.542490 & \textbf{0.681838} & 0.809441 \\
$M^{\infty}_{\mathcal{M}_{full}}$ & 0.621851 & 0.577133 & 0.526137 & 0.679054 & \textbf{0.818054} \\
\hline
\bottomrule
\end{tabular}

    }
    
    \resizebox{0.6\linewidth}{!}{%
    \begin{tabular}{|l|c|c|c|c|c|}
\toprule
\hline
 & $\mathcal{R}_{\mathcal{M}_{q-bay}, \mathcal{U}^{-\infty}}$ & $\mathcal{R}_{\mathcal{M}_{q-bay}, \mathcal{U}^{-1}}$ & $\mathcal{R}_{\mathcal{M}_{q-bay}, \mathcal{U}^{0}}$ & $\mathcal{R}_{\mathcal{M}_{q-bay}, \mathcal{U}^{1}}$ & $\mathcal{R}_{\mathcal{M}_{q-bay}, \mathcal{U}^{\infty}}$ \\
\hline
\midrule
$M^{-\infty}_{\mathcal{M}_{q-bay}}$ & \textbf{0.887444} & 0.781227 & 0.728118 & 0.728120 & 0.728125 \\
$M^{-1}_{\mathcal{M}_{q-bay}}$ & 0.872286 & 0.799687 & 0.763387 & 0.772136 & 0.789634 \\
$M^{0}_{\mathcal{M}_{q-bay}}$ & 0.866938 & \textbf{0.804483} & \textbf{0.773256} & \textbf{0.785821} & 0.810953 \\
$M^{1}_{\mathcal{M}_{q-bay}}$ & 0.802770 & 0.761888 & 0.741446 & 0.778781 & 0.853452 \\
$M^{\infty}_{\mathcal{M}_{q-bay}}$ & 0.638852 & 0.617576 & 0.606938 & 0.708914 & \textbf{0.912866} \\
\hline
\bottomrule
\end{tabular}

    }
    \vspace{10pt}
    \caption{Evaluation of the resulting explanations for each metacluster of the synthetic datasets. Each line corresponds to a decision tree trained with one specific mistakeness. Each column corresponds to the $\mathcal{U}$-evidential representativeness of the decision tree. In bold, the best decision tree for each representativeness. We can see that decision trees trained with $\lambda$-evidential mistakeness function tend to be the best in terms of $\mathcal{U}^\lambda$-evidential representativeness.}
    \label{tab:results_synthetic}
\end{table}

\begin{table}[h]
    \centering
    \resizebox{0.6\linewidth}{!}{%
    \begin{tabular}{|l|c|c|c|c|c|}
\toprule
\hline
 & $\mathcal{R}_{\mathcal{M}_{CB-2}, \mathcal{U}^{-\infty}}$ & $\mathcal{R}_{\mathcal{M}_{CB-2}, \mathcal{U}^{-1}}$ & $\mathcal{R}_{\mathcal{M}_{CB-2}, \mathcal{U}^{0}}$ & $\mathcal{R}_{\mathcal{M}_{CB-2}, \mathcal{U}^{1}}$ & $\mathcal{R}_{\mathcal{M}_{CB-2}, \mathcal{U}^{\infty}}$ \\
\hline
\midrule
$M^{-\infty}_{\mathcal{M}_{CB-2}}$ & \textbf{\textbf{0.864286}} & 0.658929 & 0.453571 & 0.458929 & 0.464286 \\
$M^{-1}_{\mathcal{M}_{CB-2}}$ & \textbf{\textbf{0.864286}} & 0.667857 & 0.471429 & 0.480357 & 0.489286 \\
$M^{0}_{\mathcal{M}_{CB-2}}$ & 0.750000 & \textbf{0.669643} & \textbf{0.589286} & 0.694643 & 0.800000 \\
$M^{1}_{\mathcal{M}_{CB-2}}$ & 0.714286 & 0.651786 & \textbf{0.589286} & \textbf{0.726786} & \textbf{\textbf{0.864286}} \\
$M^{\infty}_{\mathcal{M}_{CB-2}}$ & 0.714286 & 0.651786 & \textbf{0.589286} & \textbf{0.726786} & \textbf{\textbf{0.864286}} \\
\hline
\bottomrule
\end{tabular}

    }
    
    \resizebox{0.6\linewidth}{!}{%
    \begin{tabular}{|l|c|c|c|c|c|}
\toprule
\hline
 & $\mathcal{R}_{\mathcal{M}_{CD-2}, \mathcal{U}^{-\infty}}$ & $\mathcal{R}_{\mathcal{M}_{CD-2}, \mathcal{U}^{-1}}$ & $\mathcal{R}_{\mathcal{M}_{CD-2}, \mathcal{U}^{0}}$ & $\mathcal{R}_{\mathcal{M}_{CD-2}, \mathcal{U}^{1}}$ & $\mathcal{R}_{\mathcal{M}_{CD-2}, \mathcal{U}^{\infty}}$ \\
\hline
\midrule
$M^{-\infty}_{\mathcal{M}_{CD-2}}$ & \textbf{0.913571} & 0.780357 & 0.647143 & 0.647857 & 0.648571 \\
$M^{-1}_{\mathcal{M}_{CD-2}}$ & \textbf{0.913571} & \textbf{0.783571} & \textbf{0.653571} & 0.656071 & 0.658571 \\
$M^{0}_{\mathcal{M}_{CD-2}}$ & \textbf{0.913571} & \textbf{0.783571} & \textbf{0.653571} & 0.656071 & 0.658571 \\
$M^{1}_{\mathcal{M}_{CD-2}}$ & 0.775714 & 0.682500 & 0.589286 & \textbf{0.677500} & 0.765714 \\
$M^{\infty}_{\mathcal{M}_{CD-2}}$ & 0.661429 & 0.575000 & 0.488571 & 0.632500 & \textbf{0.776429} \\
\hline
\bottomrule
\end{tabular}

    }
    
    \resizebox{0.6\linewidth}{!}{%
    \begin{tabular}{|l|c|c|c|c|c|}
\toprule
\hline
 & $\mathcal{R}_{\mathcal{M}_{CD-4}, \mathcal{U}^{-\infty}}$ & $\mathcal{R}_{\mathcal{M}_{CD-4}, \mathcal{U}^{-1}}$ & $\mathcal{R}_{\mathcal{M}_{CD-4}, \mathcal{U}^{0}}$ & $\mathcal{R}_{\mathcal{M}_{CD-4}, \mathcal{U}^{1}}$ & $\mathcal{R}_{\mathcal{M}_{CD-4}, \mathcal{U}^{\infty}}$ \\
\hline
\midrule
$M^{-\infty}_{\mathcal{M}_{CD-4}}$ & \textbf{0.696531} & 0.523750 & 0.416378 & 0.441318 & 0.470153 \\
$M^{-1}_{\mathcal{M}_{CD-4}}$ & 0.694184 & \textbf{0.526565} & \textbf{0.422959} & 0.454137 & 0.498980 \\
$M^{0}_{\mathcal{M}_{CD-4}}$ & 0.653622 & 0.503031 & 0.414898 & \textbf{0.458236} & 0.537194 \\
$M^{1}_{\mathcal{M}_{CD-4}}$ & 0.467959 & 0.345948 & 0.242449 & 0.378104 & 0.543163 \\
$M^{\infty}_{\mathcal{M}_{CD-4}}$ & 0.421224 & 0.320476 & 0.236990 & 0.388116 & \textbf{0.601071} \\
\hline
\bottomrule
\end{tabular}

    }
    \vspace{10pt}
    \caption{Evaluation of the resulting explanations for each metacluster of the real-world datasets. Each line corresponds to a decision tree trained with one specific mistakeness. We implemented the IEMM algorithm over the datasets \texttt{Credal\_Bird-2} ($\mathcal{M}_{CB-2}$ with 2 classes), \texttt{Credal\_Dog-2} ($\mathcal{M}_{CD-2}$ with 2 classes) and \texttt{Credal\_Dog-4} ($\mathcal{M}_{CD-4}$ with 4 classes). Each column corresponds to the $\mathcal{U}$-evidential representativeness of the decision tree. In bold, the best decision tree for each representativeness. We can see that decision trees trained with $\lambda$-evidential mistakeness function tend to be the best in terms of $\mathcal{U}^\lambda$-evidential representativeness.}
    \label{tab:results}
\end{table}

\end{document}